\documentclass[runningheads]{llncs}
\usepackage[T1]{fontenc}
\usepackage{graphics} 
\usepackage{epsfig} 
\usepackage{times} 
\usepackage{amsmath} 
\usepackage{amssymb}  
\usepackage{adjustbox}
\usepackage{marvosym}

\usepackage{subfigure}
\usepackage{multirow}
\usepackage{booktabs}
\usepackage{algorithm}
\usepackage{algorithmic}

\begin{document}
\title{Robust Regularized Policy Iteration under Transition Uncertainty}
%
%
\author{Anonymous submission}
\author{Hongqiang Lin\inst{1} \and Zhenghui Fu\inst{2} \and Weihao Tang\inst{1} \and Pengfei Wang\inst{3} \and Yiding Sun\inst{4} \and Qixian Huang\inst{5}$\textsuperscript{(\Letter)}$ \and Dongxu Zhang\inst{4}$\textsuperscript{(\Letter)}$}

\authorrunning{H.Lin et al.}
%
%
\institute{College of Computer Science and Technology, Zhejiang University, Hangzhou, China \and School of Artificial Intelligence, Optics and Electronics (iOPEN), Northwestern Polytechnical University, Xi'an, China \and School of Software Technology, Zhejiang University, Hangzhou, China \and School of Software Engineering, Xi'an Jiaotong University, Xi'an, China \and School of Systems Science and Engineering, Sun Yat-sen University, Guangzhou, China}
%
\maketitle              
\begin{abstract}
Offline reinforcement learning (RL) enables data-efficient and safe policy learning without online exploration, but its performance often degrades under distribution shift. The learned policy may visit out-of-distribution state-action pairs where value estimates and learned dynamics are unreliable. To address policy-induced extrapolation and transition uncertainty in a unified framework, we formulate offline RL as robust policy optimization, treating the transition kernel as a decision variable within an uncertainty set and optimizing the policy against the worst-case dynamics. We propose Robust Regularized Policy Iteration (RRPI), which replaces the intractable max-min bilevel objective with a tractable KL-regularized surrogate and derives an efficient policy iteration procedure based on a robust regularized Bellman operator. We provide theoretical guarantees by showing that the proposed operator is a $\gamma$-contraction and that iteratively updating the surrogate yields monotonic improvement of the original robust objective with convergence. Experiments on D4RL benchmarks demonstrate that RRPI achieves strong average performance, outperforming recent baselines including percentile-based methods on the majority of environments while remaining competitive on the rest. Moreover, RRPI exhibits robust performance by aligning lower $Q$-values with high epistemic uncertainty, which prevents the policy from executing unreliable out-of-distribution actions.
\end{abstract}
\section{Introduction}
Reinforcement learning (RL) has become a key ingredient of modern decision-making systems, driving breakthroughs in complex domains such as strategic game playing, embodied intelligence, and large language model reasoning by optimizing long-horizon policies from interaction data \cite{10.5555/3312046,liu2025optimized,zhang2026queriesneeddeepthought}. Despite these successes, standard online RL typically relies on trial-and-error exploration, a process that can be prohibitively costly in high-stakes real-world applications \cite{gottesman2019guidelines,kumar2022pre}. This has driven growing interest in offline RL \cite{levine2020offline}, which aims to learn high-performing policies solely from pre-collected datasets without additional environment interaction.

The central difficulty of offline RL is distribution shift: the learned policy may query state-action pairs that lie outside the support of the dataset. In such out-of-distribution (OOD) regions, value estimates can suffer from severe extrapolation errors, closely tied to epistemic uncertainty induced by limited coverage \cite{lurevisiting}. To mitigate these failures, existing methods either adopt conservative value learning that explicitly down-weights OOD actions \cite{kumar2020conservative,yu2021combo,huang2024efficient,NEURIPS2024_cdc1d08e} or leverage uncertainty estimates to steer the policy away from high-uncertainty regions \cite{yu2020mopo,Qiao_Lyu_Jiao_Liu_Li_2025}. However, recent evidence suggests that such mechanisms may be overly conservative and can sacrifice performance even in well-supported regions \cite{park2024value}. More fundamentally, these approaches typically plan under a single learned dynamics model and thus do not directly capture uncertainty in the transition dynamics itself.

In this work, we pursue a robust policy optimization perspective. Instead of treating the learned dynamics as a point estimate, we regard the transition model as a decision variable that lies in an uncertainty set of plausible kernels, and we seek a policy that maximizes performance under the worst-case dynamics in this set \cite{dong2024online}. This formulation incorporates dynamics uncertainty directly into the objective and yields a principled alternative to heuristic uncertainty penalties. Nevertheless, the induced max-min objective leads to a challenging bilevel optimization problem, whose exact solution is often computationally prohibitive in practical offline RL pipelines. To address this, we propose \emph{Robust Regularized Policy Iteration} (RRPI), which introduces a KL-regularized surrogate objective and an associated robust regularized Bellman operator, enabling an efficient iterative algorithm with theoretical guarantees.
\begin{itemize}
	
	\item Building on the robust formulation, RRPI naturally accounts for transition uncertainty and provides worst-case guarantees. By introducing a regularizer, RRPI derives a surrogate objective that can be solved iteratively, thereby avoiding the prohibitive cost of directly solving the original max--min bilevel problem.
	\item We provide theoretical results showing that optimizing the surrogate objective improves the original robust objective, and that the RRPI iterates converge under mild conditions.
	\item Empirically, RRPI surpasses recent state-of-the-art baselines on D4RL benchmarks and exhibits strong robustness. In particular, the learned policy attains lower returns in regions with higher epistemic uncertainty, indicating that RRPI avoids unreliable out-of-distribution behaviors under transition uncertainty.
\end{itemize}

This work belongs to the family of model-based offline RL methods, which first learn a transition model from the offline dataset and then perform policy optimization by planning or dynamic programming under the learned dynamics. In contrast to approaches that plan under a single estimated model, we adopt a robust perspective and optimize against an uncertainty set over transition kernels.

\section{Related work and background}

\subsection{Related work}
Offline RL learns a policy solely from a fixed dataset, and its core challenge arises from distributional shift: the learned policy may query state-action pairs that are poorly covered by the data, making value estimates unreliable \cite{levine2020offline,prudencio2023survey} and leading to extrapolation error \cite{park2025model}.  Uncertainty in offline RL is often characterized as epistemic uncertainty induced by limited data support, which differs from aleatoric uncertainty originating from stochastic environment dynamics and observation noise \cite{guo2022model,ghosh22a,lin2026offlinepolicyoptimizationposterior,lin2026regularizedofflinepolicyoptimization}. To mitigate this epistemic uncertainty, existing approaches primarily adopt conservative value learning that penalizes OOD actions \cite{kumar2020conservative,yu2021combo,huang2024efficient,NEURIPS2024_cdc1d08e}, regularized policy updates that constrain deviations from the behavior policy \cite{Wu2019BehaviorRO,adaptive2023zhang}, or explicit uncertainty estimation via ensembles or Bayesian methods \cite{rigter2022rambo,guo2022model,Qiao_Lyu_Jiao_Liu_Li_2025}.
In contrast to prior methods, we cast offline RL as a robust optimization problem over an uncertainty set, which incorporates dynamics uncertainty directly without requiring explicit uncertainty estimation. In this formulation, the dynamics model is treated as a variable rather than a fixed quantity. This method avoids heuristic penalties on OOD action values and does not require constraining the learned policy to remain close to the behavior policy. Building on this, we develop an iterative policy optimization algorithm with theoretical guarantees of convergence to an optimal solution.

\subsection{Background}
\subsubsection{Markov Decision Process (MDP).}
A Markov decision process (MDP) is defined by the tuple ${M}=( S,  A, p, r, \rho_0, \gamma)$, where $S$ and $A$ denote the state and action spaces, respectively. 
The transition kernel $p(\cdot\mid s,a)\in \Delta(S)$ specifies the distribution over next states given $(s,a)$, where $\Delta(\cdot)$ denotes the probability simplex. 
The reward function is $r:S \times A \to \mathbb{R}$, $\rho_0$ is the initial-state distribution, and $\gamma\in(0,1)$ is the discount factor. 
The goal is to learn a policy $\pi(\cdot\mid s)\in \Delta(A)$ that maximizes the expected discounted return:
\begin{equation}
	\label{eq:obj}
	\eta(\pi,p)=\mathop{\mathbb E}\limits_{\rho_0,\pi,p}\bigg[\sum_{t=0}^{\infty}\gamma^tr(s_t,a_t)\bigg].
\end{equation}
\subsubsection{Model-Based Offline RL.}
In offline RL, we perform policy optimization using only a fixed dataset $D$ collected by an unknown behavior policy. A nominal dynamics model $\hat p$ is typically learned by maximizing the log-likelihood:
\begin{equation}
	\mathop{\mathbb E}\limits_{(s_t,a_t,r_{t+1},s_{t+1})\sim D}\big[\log \hat{p}(s_{t+1},r_{t+1}\mid s_t,a_t)\big].
\end{equation}

Subsequently, an offline RL algorithm can be applied to a composite dataset $D \cup D_{\text{model}}$, where $D_{\text{model}}$ contains synthetic transitions generated by rolling out $\hat{p}$.

\subsubsection{Robust MDP.}
Robust MDPs model transition uncertainty by treating the dynamics $p$ as a decision variable rather than a fixed estimate. Given an uncertainty set $P$ that contains plausible transition kernels, robust MDP seeks a policy that maximizes performance under the worst-case dynamics in $P$ \cite{dong2024online}. Concretely, the robust optimal policy is obtained by solving the following max-min problem:
\begin{equation}
	\label{eq:robust_mdp}
	\pi^*=\arg\max_{\pi\in\Pi}\min_{p\in P}\eta(\pi,p).
\end{equation}
We denote the corresponding robust return by $J(\pi)=\min_{p\in P}\eta(\pi,p)$. This objective provides a principled way to hedge against model misspecification. In particular, when a policy visits state--action pairs with limited data coverage, the learned dynamics model can extrapolate and lead to compounding errors over rollouts. Optimizing $\eta(\pi,\hat p)$ for a point estimate $\hat p$ can therefore be brittle under distribution shift. In contrast, the robust objective explicitly accounts for adverse, yet plausible, transitions, yielding policies that are less sensitive to transition errors.

The modeling power and conservatism of robust MDPs depend critically on how $P$ is constructed. A common assumption is \emph{rectangularity}, where $P$ factorizes across states or state-action pairs, enabling the inner minimization to decompose locally and making robust dynamic programming tractable. While convenient, rectangular sets can be overly conservative in high-dimensional problems since they allow independent worst-case perturbations at each $(s,a)$. In practice, $P$ is often built from function approximation, for example by using model ensembles, which capture epistemic uncertainty induced by limited coverage.

Unlike standard offline RL methods that optimize the return under a single learned dynamics (implicitly assuming the model is accurate), the robust formulation explicitly optimizes against transition uncertainty. As a result, it naturally down-weights actions whose long-term consequences are highly sensitive to transition errors, which is especially desirable in offline settings. Figure~\ref{fig2:robust} illustrates the robust optimization viewpoint.
\begin{figure}[t!]
	\centering
	\includegraphics[width=0.49\textwidth]{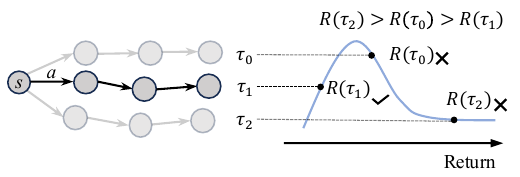}
	\caption{Illustration of robust optimization. We roll out different dynamics models from the uncertainty set $P$ to obtain three trajectories $\tau_0$, $\tau_1$, and $\tau_2$, whose returns satisfy $R(\tau_2)>R(\tau_0)>R(\tau_1)$. Accordingly, we optimize with respect to the worst-case trajectory $\tau_1$ (i.e., $R(\tau_1)$).}
	\label{fig2:robust}
\end{figure}
\section{Methods}
In this section, we first present our policy optimization algorithm and then describe its implementation details.
\subsection{Robust regularized policy iteration}
Directly optimizing Eq.~\eqref{eq:robust_mdp} is computationally inefficient, as it requires accurately evaluating the inner subproblem before computing a policy gradient to guide the outer-loop update. 
To improve efficiency, we draw inspiration from TRPO \cite{pmlr-v37-schulman15} and construct a surrogate objective $\widehat{\eta}(\pi,p,\mu)$. 
The surrogate is designed to satisfy two key requirements: (i) $\widehat{\eta}(\pi,p,\mu)$ is tractable to optimize with standard policy optimization method, and (ii) improving the surrogate objective yields a monotonic improvement in the original objective ${\eta}(\pi,p)$. We define $\widehat{\eta}(\pi,p,\mu)$ as
\begin{equation}
	\label{eq:reg_obj}
	\widehat{\eta}(\pi,p,\mu)=\mathop{\mathbb E}\limits_{\rho_0,\pi,p}\bigg[\sum_{t=0}^{\infty}\gamma^t(r(s_t,a_t)-\alpha\log\frac{\pi}{\mu})\bigg],
\end{equation}
where $\alpha>0$ is the regularization coefficient and $\mu$ is the reference policy.

Replacing the objective in Eq.~\eqref{eq:robust_mdp} with $\widehat{\eta}(\pi,p,\mu)$ yields a surrogate max-min problem. To solve it, we introduce the following robust regularized Bellman operator:
\begin{equation}
	\label{eq:operator}
	\mathcal{T}Q(s,a)=r(s,a)+\gamma V(s'),
\end{equation}
where
\begin{equation}
	V(s')=\min_{p\in P}\mathbb{E}_p\left[\alpha\log \mathbb{E}_\mu\exp\!\left(\frac{1}{\alpha}Q(s',a')\right)\right].	
\end{equation}

We next analyze theoretical properties of the operator $\mathcal{T}$.
\begin{theorem}
	\label{thm:rrpo}
	The robust regularized Bellman operator $\mathcal{T}$ is a $\gamma$-contraction mapping under the $\|\cdot\|_{\infty}$ norm. Starting from any initial action-value function $Q:S\times A\to\mathbb{R}$ and repeatedly applying $\mathcal{T}$, the resulting sequence converges to a fixed point $Q^*$. Moreover, the optimal robust policy for the surrogate objective $\widehat{\eta}(\pi,p,\mu)$ is given by
	\begin{equation}
		\label{eq:opt_policy}
		\pi^*(a|s)\propto \mu(a|s)\exp\!\left(\frac{1}{\alpha}Q^*(s,a)\right).
	\end{equation}
\end{theorem}

\begin{remark}
	The operator in Eq.~\eqref{eq:operator} is closely related to KL-regularized dynamic programming. In particular, for any next-state distribution induced by  dynamics model $p$, the term $\alpha\log \mathbb{E}_{a'\sim \mu(\cdot|s')}\left[\exp\left(Q(s',a')/\alpha\right)\right]$ is the convex conjugate of the KL divergence and corresponds to eliminating the inner maximization over policies in
	
	$$\max_{\pi} \left\{\mathop{\mathbb{E}}_{a'\sim \pi(\cdot|s')}\!\left[Q(s',a')\right]-\alpha D_\mathrm{KL}\left(\pi(\cdot|s')\,\|\,\mu(\cdot|s')\right)\right\}.$$
	
	Consequently, the induced optimal policy in Eq.~\eqref{eq:opt_policy} takes a Boltzmann form that is absolutely continuous with respect to the reference policy $\mu$.
	Moreover, the contraction factor remains $\gamma$ because the $\log$-$\exp$ mapping is non-expansive under $\|\cdot\|_\infty$ when $\alpha>0$, and the minimization over $p\in P$ preserves monotonicity.
	The coefficient $\alpha$ controls the strength of regularization. Smaller $\alpha$ makes the update more greedy with respect to $Q^*$, while larger $\alpha$ yields a more conservative update that stays closer to $\mu$. This perspective helps explain why RRPI can be stabilized by keeping successive policies close in KL, while still allowing the algorithm to gradually move toward the soft-greedy solution implied by the robust regularized Bellman operator.
\end{remark}

Although the operator $\mathcal{T}$ enables efficient optimization of the surrogate objective $\widehat{\eta}(\pi,p,\mu)$, the surrogate generally differs from the original return $\eta(\pi,p)$. The following theorem characterizes how iteratively updating the reference policy along the optimization trajectory links improvements in the surrogate problem to monotonic progress on the original robust objective.
\begin{theorem}
	\label{thm:irpo}
	Starting from any stochastic initial policy $\pi_0:S\to\Delta(A)$ with full support (i.e., $\pi_0(a|s)>0$ for all $(s,a)\in S\times A$), we generate a sequence $\{\pi_i\}_{i\ge 0}$ by repeatedly solving the robust regularized subproblem with the reference policy set to the previous iterate:
	\begin{equation}
		\pi_{i+1}\in\arg\max_{\pi\in\Pi}\min_{p\in P}\widehat{\eta}(\pi,p,\mu=\pi_i).
	\end{equation}
	Then the sequence monotonically improves the original objective $J(\pi)\:=min_{p\in P}\eta(\pi,p)$ in the sense that $J(\pi_{i+1})\ge J(\pi_i)$ for all $i\ge 0$.
	Furthermore, the sequence converges to an optimal robust policy for the original (non-regularized) problem: for any state $s$ and actions $a,a'\in A$ satisfying
	$\lim_{i\to\infty}Q_{P}^{\pi_i}(s,a)>\lim_{i\to\infty}Q_{P}^{\pi_i}(s,a')$,
	we have $\pi_i(a\mid s)/\pi_i(a'\mid s)\to\infty$ as $i\to\infty$, where $Q_{P}^{\pi_i}$ denotes the action-value function under the worst-case dynamics in $P$.
\end{theorem}
\begin{remark}
	Theorem~\ref{thm:irpo} provides a formal link between optimizing the regularized surrogate $\widehat{\eta}(\pi,p,\mu)$ and improving the original objective $J(\pi)=\min_{p\in P}\eta(\pi,p)$. A key ingredient is the reference-policy update $\mu=\pi_i$, which makes each subproblem a local improvement step around the current iterate. This improvement step mirrors the role of trust regions while accounting for worst-case transition uncertainty through the inner minimization over $p\in P$. The requirement that $\pi_0$ has full support ensures that all subsequent iterates admit a well-defined regularization term $\log\frac{\pi}{\mu}$ (with $\mu=\pi_i$) and prevents irrevocably assigning zero probability to actions. This is important in robust settings, where an action that is suboptimal under nominal dynamics may become optimal under a different adversarial model in $P$. The ratio condition $\pi_i(a|s)/\pi_i(a'|s)\to\infty$ whenever the limiting action-value of $a$ exceeds that of $a'$ indicates that the sequence becomes increasingly greedy with respect to the limiting $Q$-function, thereby recovering an optimal robust policy of the original problem in the limit.
	
\end{remark}
\begin{algorithm}[t!]
	\caption{Robust Regularized Policy Iteration} 
	\label{algo} 
	\begin{algorithmic}[1]
		\REQUIRE Dataset $D$, model ensemble size $N$.
		\STATE \textbf{Initialization:} Randomly initialize $Q$-function $Q_{\theta}(s,a)$ and policy $\pi_{\phi}(a|s)$. Initialize target $Q$-function $Q_{\theta'}(s,a)$ and reference policy $\mu_{\phi'}(a|s)$ with $\theta'\gets \theta$, $\phi'\gets\phi$. Randomly initialize $N$ dynamics models $\{ p_{\nu_i}(s'|s,a) \}_{i=1}^N$, forming the model ensemble.
		\STATE Train each transition dynamics model $p_{\nu_i}(s'|s,a)$ to maxmize: $$\mathop{\mathbb E}\limits_{(s_t,a_t,r_{t+1},s_{t+1})\sim \mathcal{D}}[\log p_{\nu_i}(s_{t+1},r_{t+1}|s_t,a_t)].$$
		\FOR{$i=1,2,\cdots ,G$}
		\STATE Perform policy evaluation using Eq.~\eqref{eq:operator} and update $Q_{\theta}$ to approximate the optimal action-value function.
		\STATE Perform policy improvement using Eq.~\eqref{eq:pi_loss} and update the policy $\pi_{\phi}(a|s)$.
		
		\STATE Update target $Q$-function and reference policy $\mu$.
		
		\ENDFOR
		\STATE \textbf{Return:}  $\pi_{\phi}$.
	\end{algorithmic}
\end{algorithm}

\subsection{Implementations}

Algorithm~\ref{algo} summarizes the proposed procedure. We parameterize the action-value function $Q_{\theta}$ and the policy $\pi_{\phi}$ with neural networks. Constructing a suitable uncertainty set $P$ is nontrivial in practice: it must be rich enough to cover model errors while remaining computationally tractable for the inner minimization. In our implementation, we approximate $P$ using a learned dynamics model ensemble, where each transition model $p_{\nu_i}$ parameterizes a Gaussian distribution over next states and rewards \cite{yu2020mopo,kidambi2020morel,guo2022model}.
Specifically, for each $(s,a)$, the ensemble induces a set of plausible next-state distributions, yielding a tractable approximation of the uncertainty set. We instantiate the inner minimization by selecting, among the ensemble members, the transition model that yields the smallest one-step robust target in the Bellman backup. This worst-case member approximation serves as a computationally lightweight proxy to robust dynamic programming over $P$ and directly penalizes state-action pairs where the models disagree. Intuitively, when the dataset provides poor coverage, different ensemble members extrapolate differently, effectively enlarging the uncertainty set and causing the robust backup to propagate pessimistic value estimates. In contrast, in well-covered regions the ensemble predictions concentrate, and the robust backup reduces to near-nominal evaluation.

For policy evaluation, we update $Q_{\theta}$ by minimizing the Bellman residual 
\begin{equation}
	\begin{aligned}
		L_1(\theta)=&\quad\omega_1\mathop{\mathbb{E}}_{(s,a)\sim D}\bigl|Q_{\theta}(s,a)-\mathcal{T}Q_{\theta}(s,a)\bigr|^2 \\& 
		+\quad\omega_2\mathop{\mathbb{E}}_{(s,a,s')\sim D_\text{model}}\bigl|Q_{\theta}(s,a)-\mathcal{T}Q_{\theta}(s,a)\bigr|^2,
	\end{aligned}
\end{equation}
where $\omega_1$ and $ \omega_2$ denote the sampling weights for transitions drawn from the offline dataset and the model-generated buffer, respectively.
\begin{table*}[ht!]
	\centering
	\caption{Performance comparison on D4RL datasets. Scores are normalized as $(\text{score} - \text{random}) / (\text{expert} - \text{random})$ and presented as mean $\pm$ standard deviation. Our results are averaged across 4 random seeds.}
	\label{tab:d4rl_performance}
	\setlength{\tabcolsep}{2pt} 
	\adjustbox{width=0.99\textwidth}
	{
		\begin{tabular}{l|cccccccc}
			\toprule[1.6pt]
			Dataset & CQL & DMG & EPQ & MOReL & RAMBO & PMDB & AMG & RRPI \\
			\midrule
			HalfCheetah-Random        & 31.3$\pm$3.5 & 28.8$\pm$1.3 & 33.0$\pm$2.4 & 38.9$\pm$1.8 & 39.5$\pm$3.5 & 37.8$\pm$0.2 & \textbf{45.4}$\pm$2.8 & 35.5$\pm$0.5\\
			HalfCheetah-Medium        & 46.9$\pm$0.4 & 54.9$\pm$0.2 & 67.3$\pm$0.5 & 60.7$\pm$4.4 & \textbf{77.9}$\pm$4.0 & 75.6$\pm$1.3 & 72.2$\pm$0.6 & 75.2$\pm$0.7\\
			HalfCheetah-Expert        & 97.3$\pm$1.1 & 95.9$\pm$0.3 & \textbf{107.2}$\pm$0.2 & 8.4$\pm$11.8 & 79.3$\pm$15.1 & 105.7$\pm$1.0 & 89.4$\pm$26.4 & 90.7$\pm$1.1\\
			HalfCheetah-Medium-Expert & 95.0$\pm$1.4 & 91.1$\pm$4.2 & 95.7$\pm$0.3 & 80.4$\pm$11.7 & 95.4$\pm$5.4 & \textbf{108.5}$\pm$0.5 & 103.7$\pm$0.2 & 105.3$\pm$1.7\\
			HalfCheetah-Medium-Replay & 45.3$\pm$0.3 & 51.4$\pm$0.3 & 62.0$\pm$1.6 & 44.5$\pm$5.6 & 68.7$\pm$5.3 & 71.7$\pm$1.1 & 67.6$\pm$3.4 & \textbf{74.4}$\pm$0.7\\
			HalfCheetah-Full-Replay   & 76.9$\pm$0.9 & 79.9$\pm$1.2 & 85.3$\pm$0.7 & 70.1$\pm$5.1 & 87.0$\pm$3.2 & 90.0$\pm$0.8 & 86.3$\pm$1.7 & 83.7$\pm$0.9\\
			\midrule
			Hopper-Random        & 5.3$\pm$0.6  & 20.4$\pm$10.4 & 32.1$\pm$0.3 & \textbf{38.1}$\pm$10.1 & 25.4$\pm$7.5 & 32.7$\pm$0.1 & 32.7$\pm$0.2 & 35.0$\pm$0.6\\
			Hopper-Medium        & 61.9$\pm$6.4 & 100.6$\pm$1.9 & 101.3$\pm$0.2 & 84.0$\pm$17.0 & 87.0$\pm$15.4 & 106.8$\pm$0.2 & 107.4$\pm$0.6 & \textbf{109.4}$\pm$0.6\\
			Hopper-Expert        & 106.5$\pm$9.1 & 111.5$\pm$2.2 & 112.4$\pm$0.5 & 80.4$\pm$34.9 & 50.0$\pm$8.1 & 111.7$\pm$0.3 & 102.3$\pm$11.9 & \textbf{114.8}$\pm$0.9\\
			Hopper-Medium-Expert & 96.9$\pm$15.1 & 110.4$\pm$3.4 & 108.8$\pm$5.2 & 105.6$\pm$8.2 & 88.2$\pm$20.5 & 111.8$\pm$0.6 & \textbf{112.7}$\pm$0.3 & 111.9$\pm$0.3\\
			Hopper-Medium-Replay & 86.3$\pm$7.3 & 101.9$\pm$1.4 & 97.8$\pm$1.0 & 81.8$\pm$17.0 & 99.5$\pm$4.8 & 106.2$\pm$0.6 & 104.4$\pm$0.4 & 
			\textbf{106.6}$\pm$0.3\\
			Hopper-Full-Replay   & 101.9$\pm$0.6 & 106.4$\pm$1.1 & 108.5$\pm$0.6 & 94.4$\pm$20.5 & 105.2$\pm$2.1 & \textbf{109.1}$\pm$0.2 & 108.5$\pm$0.7 & 108.6$\pm$0.2\\
			\midrule
			Walker2d-Random        & 5.4$\pm$1.7  & 4.8$\pm$2.2  & 23.0$\pm$0.7 & 16.0$\pm$7.7 & 0.0$\pm$0.3 & 21.8$\pm$0.1 & 22.2$\pm$0.2 & \textbf{23.7}$\pm$0.6\\
			Walker2d-Medium        & 79.5$\pm$3.2 & 92.4$\pm$2.7 & 87.8$\pm$2.1 & 72.8$\pm$11.9 & 84.9$\pm$2.6 & 94.2$\pm$1.1 & 93.2$\pm$1.1 & \textbf{97.5}$\pm$0.7\\
			Walker2d-Expert        & 109.3$\pm$0.1 & 114.7$\pm$0.4 & 109.8$\pm$1.0 & 62.6$\pm$29.9 & 1.6$\pm$2.3 & \textbf{115.9}$\pm$1.9 & 5.5$\pm$1.3 & 111.2$\pm$0.4\\
			Walker2d-Medium-Expert & 109.1$\pm$0.2 & 114.4$\pm$0.7 & 112.0$\pm$0.6 & 107.5$\pm$5.6 & 56.7$\pm$39.0 & 111.9$\pm$0.2 & 114.9$\pm$0.3 & \textbf{115.7}$\pm$1.4\\
			Walker2d-Medium-Replay & 76.8$\pm$10.0 & 89.7$\pm$5.0 & 85.3$\pm$1.0 & 40.8$\pm$20.4 & 89.2$\pm$6.7 & 79.9$\pm$0.2 & 95.6$\pm$2.1 & \textbf{96.4}$\pm$1.0\\
			Walker2d-Full-Replay   & 94.2$\pm$1.9 & 97.5$\pm$4.6 & \textbf{107.4}$\pm$0.6 & 84.8$\pm$13.1 & 88.3$\pm$4.9 & 95.4$\pm$0.7 & 99.9$\pm$3.6 & \textbf{107.3}$\pm$0.4\\
			\midrule
			\textbf{Average} & 73.7 & 81.5 & 85.4 & 65.1 & 68.0 & 88.2 & 81.3 & \textbf{89.1}\\
			\bottomrule[1.6pt]
		\end{tabular}
	}
\end{table*}

During policy improvement, we update $\pi_{\phi}$ by minimizing a KL divergence between the current policy and an unnormalized Boltzmann target induced by the critic,
\begin{equation}
	\label{eq:pi_loss}
	\begin{aligned}
		L_2(\phi)=\mathop{\mathbb{E}}\limits_{s\sim D\cup D_\text{model}}\left[D_{\mathrm{KL}}\left(\pi^*(\cdot|s),\Bigg\|\,\pi_{\phi}(\cdot|s)\right)\right],
	\end{aligned}
\end{equation}
where $\pi^*(\cdot|s)=\pi_{\phi'}(\cdot|s)\exp\!\left(\frac{1}{\alpha}Q_{\theta}(s,\cdot)\right)/Z_{\phi'}(s)$ denotes the Boltzmann target, $\pi_{\phi'}$ is the reference policy, and $Z_{\phi'}(s)=\int_A \pi_{\phi'}(a|s)\exp\!\left(\frac{1}{\alpha}Q_{\theta}(s,a)\right)\,da$ is the corresponding partition function. This KL objective can be viewed as a practical instantiation of Eq.~\eqref{eq:opt_policy}: it drives $\pi_{\phi}$ toward the soft-greedy policy induced by $Q_{\theta}$ while keeping the update close to $\pi_{\phi'}$, thereby improving stability in offline training. Additionally, we find that clipping the $Q_\theta$, that is, constraining it to be upper-bounded by a constant $\frac{1}{1-\gamma}$, helps prevent overestimation in the early stage of training and thereby stabilizes the update process.

The algorithm involves computing expectations. We approximate them using Monte Carlo estimators. Concretely, we use standard Monte Carlo sampling for expectations over states, and importance sampling for expectations over actions \cite{guo2022model}. We perform importance sampling by drawing actions from the current policy $\pi_{\phi}$:
\begin{equation}
	\begin{aligned}
		\mathbb{E}_{\pi_{\phi'}}\bigg[\exp\big(\frac{1}{\alpha}Q(s,a)&\big)\bigg]=\mathbb{E}_{\pi_{\phi}}\bigg[\frac{\pi_{\phi'}}{\pi_{\phi}}\exp\big(\frac{1}{\alpha}Q(s,a)\big)\bigg] \\&
		\approx\frac{1}{n}\sum_{i=1}^n\bigg[\frac{\pi_{\phi'}}{\pi_{\phi}}\exp\big(\frac{1}{\alpha}Q(s,a)\big)\bigg].
	\end{aligned}
\end{equation}

Intuitively, because the model ensemble captures most of the available information about the environment and we explicitly optimize against worst-case dynamics, sampling actions from the current policy encourages exploration without introducing additional approximation bias, which in turn improves performance.
\begin{figure*}[t!]
	\centering
	\includegraphics[width=0.24\textwidth]{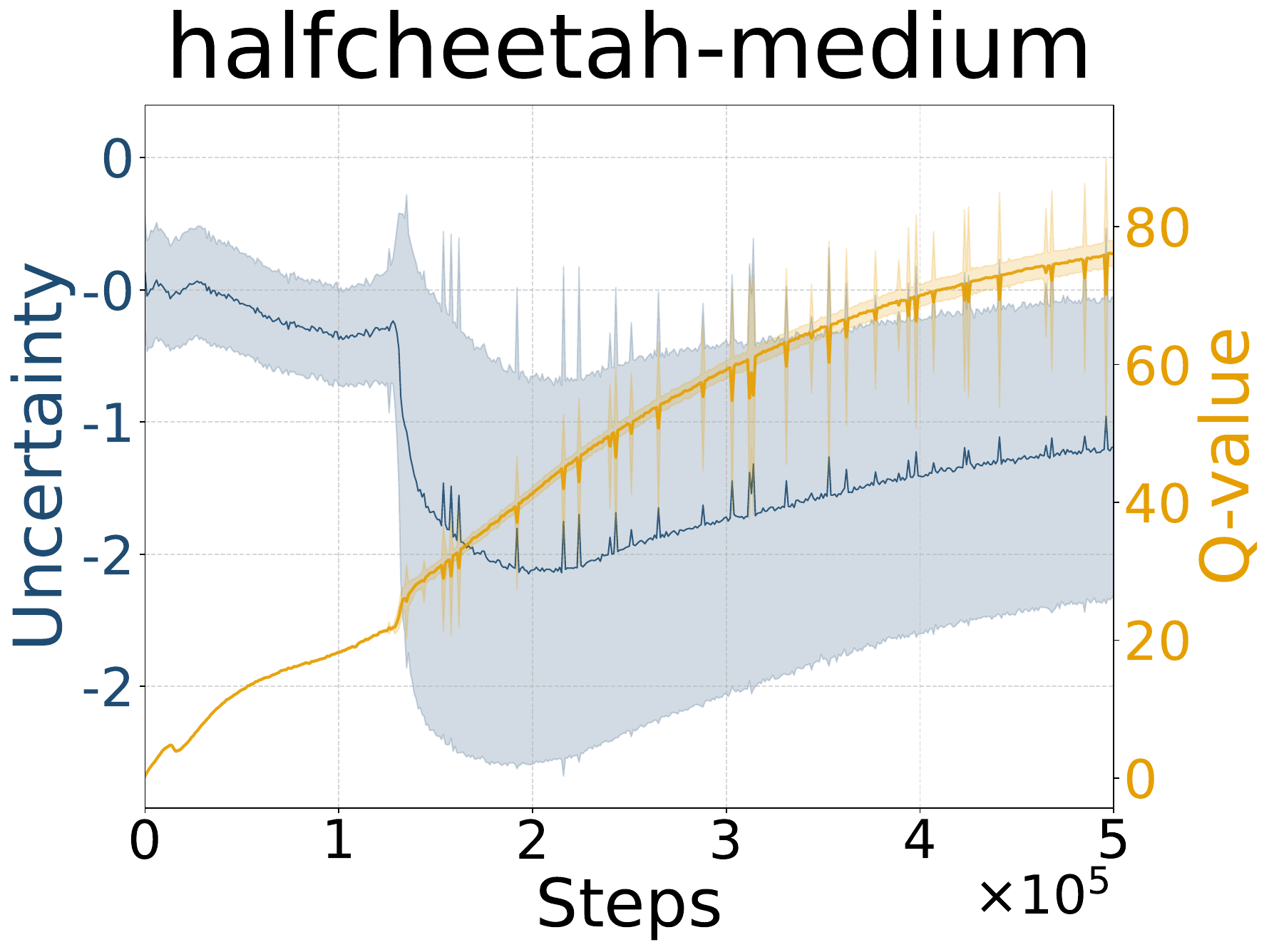}
	\includegraphics[width=0.24\textwidth]{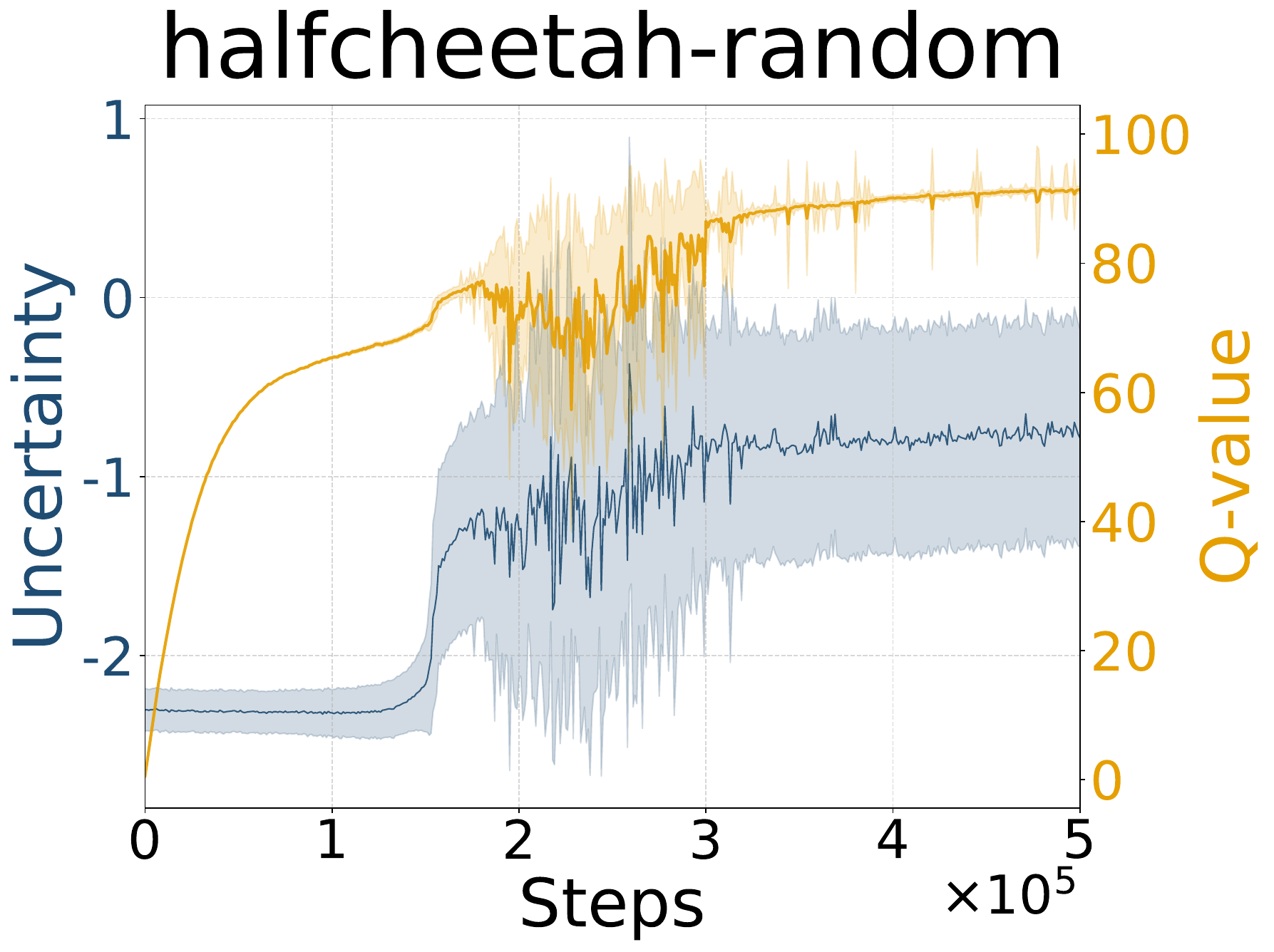}
	\includegraphics[width=0.24\textwidth]{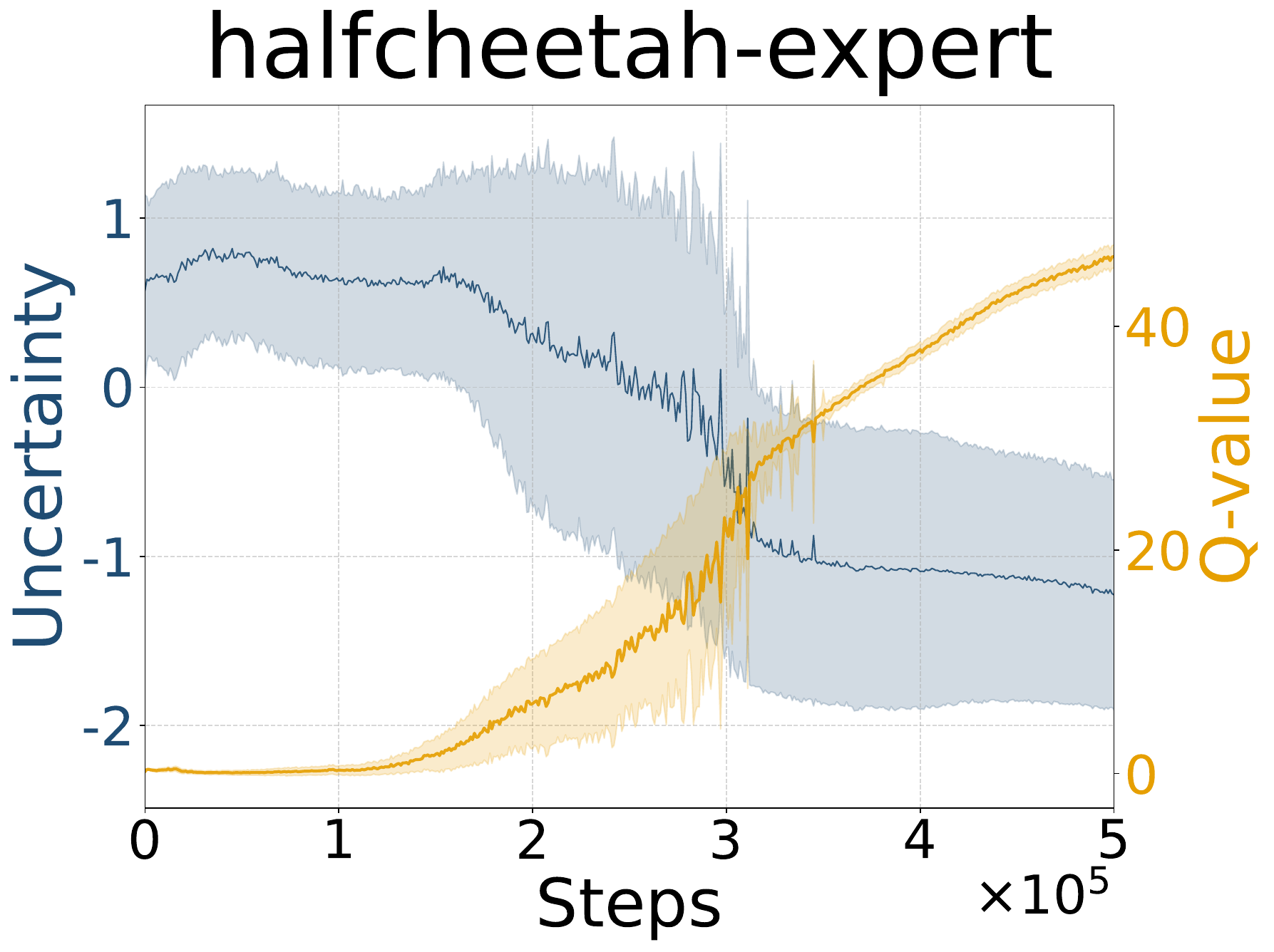}
	\includegraphics[width=0.24\textwidth]{./plots/halfcheetah-expert_qvalue.pdf}\\
	\includegraphics[width=0.24\textwidth]{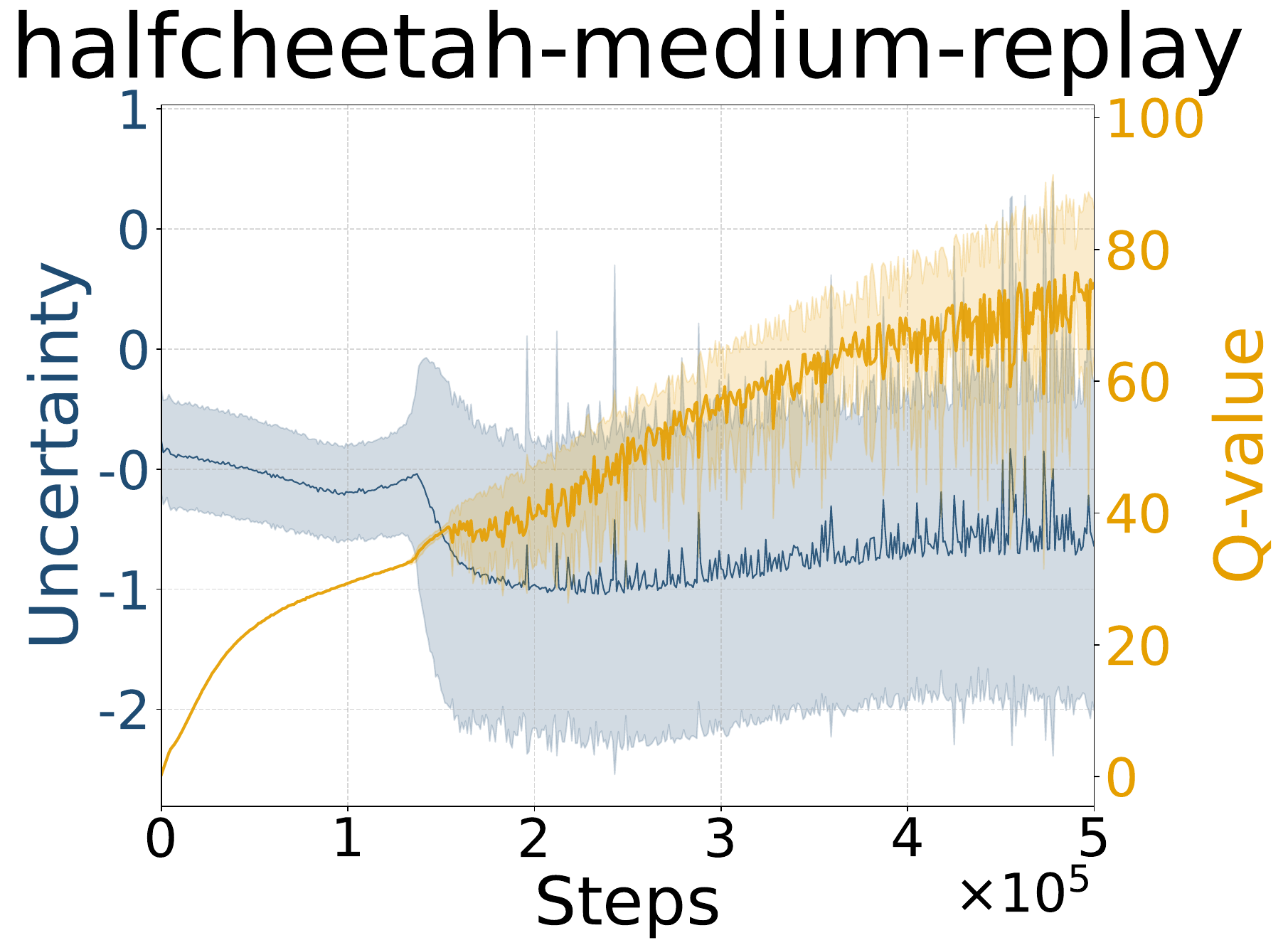}
	\includegraphics[width=0.24\textwidth]{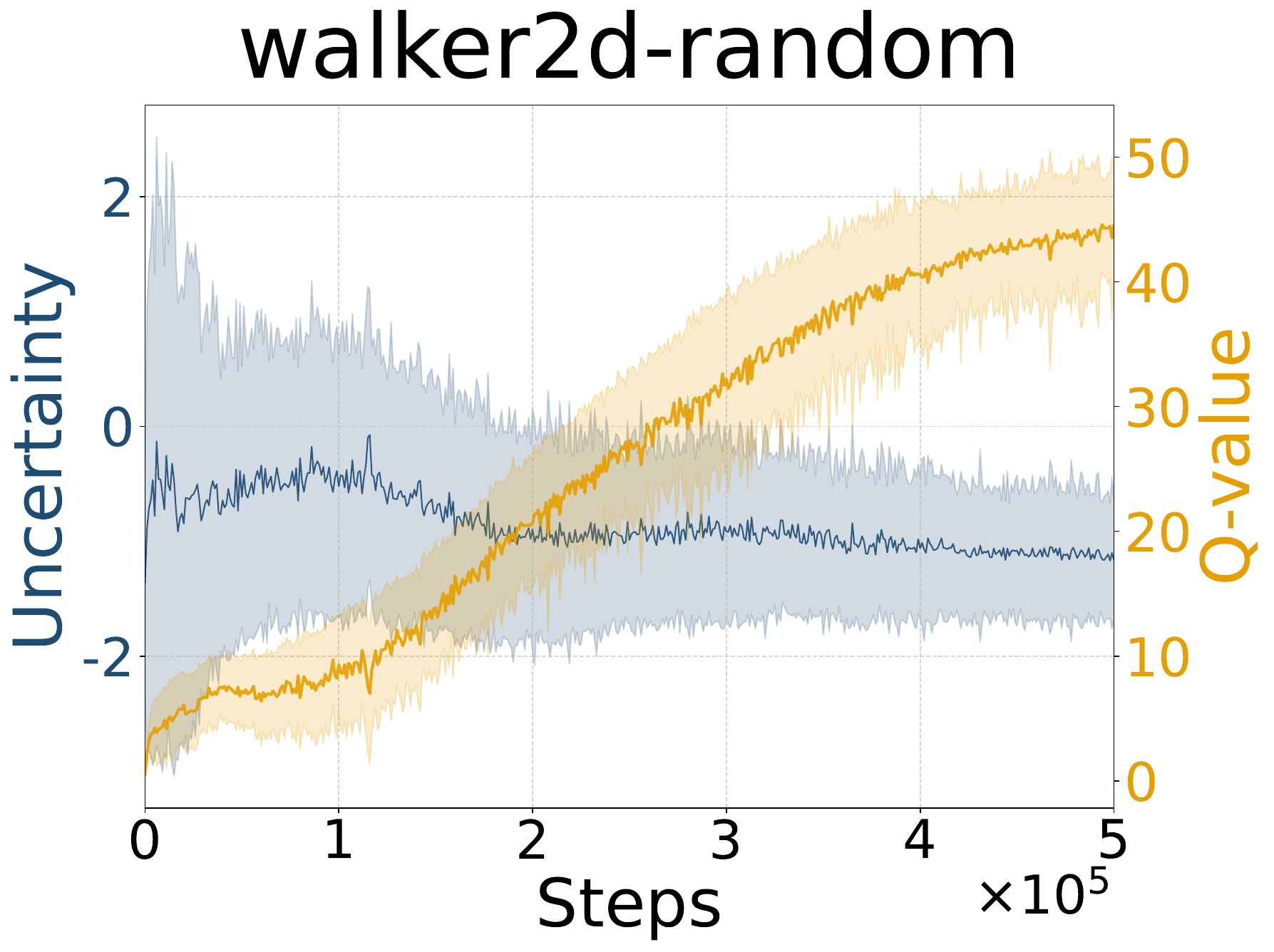}
	\includegraphics[width=0.24\textwidth]{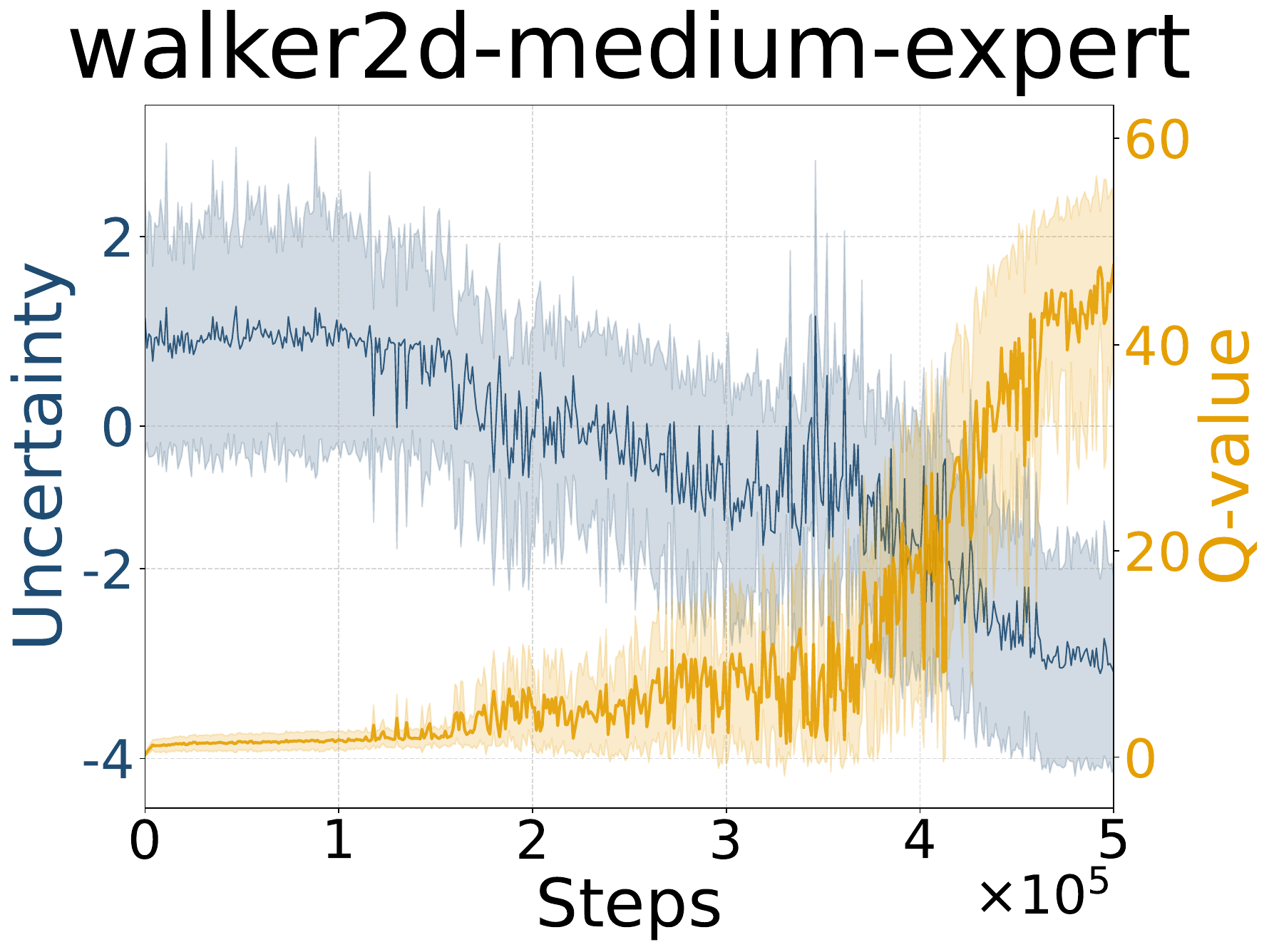}
	\includegraphics[width=0.24\textwidth]{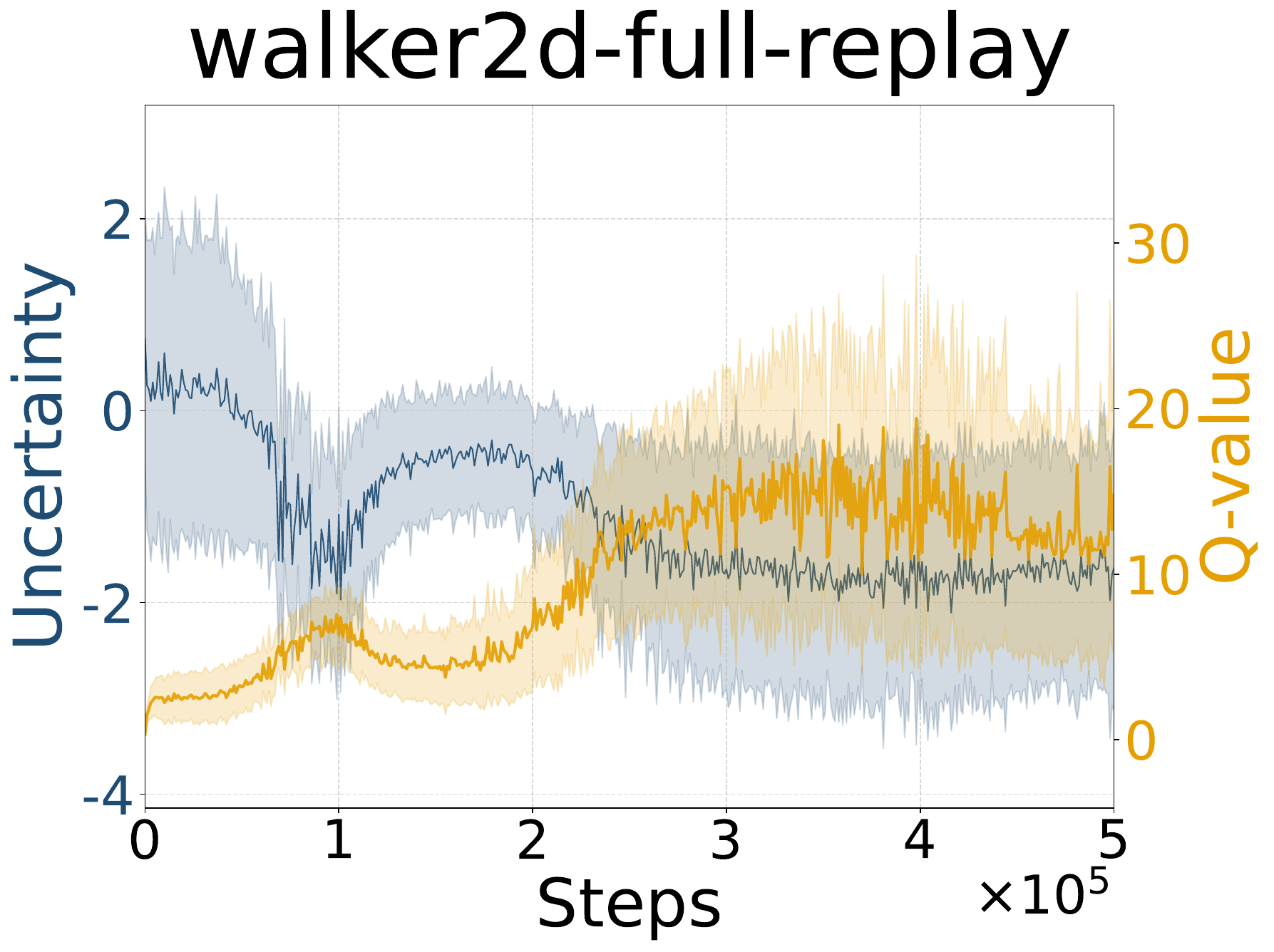}\\
	\includegraphics[width=0.24\textwidth]{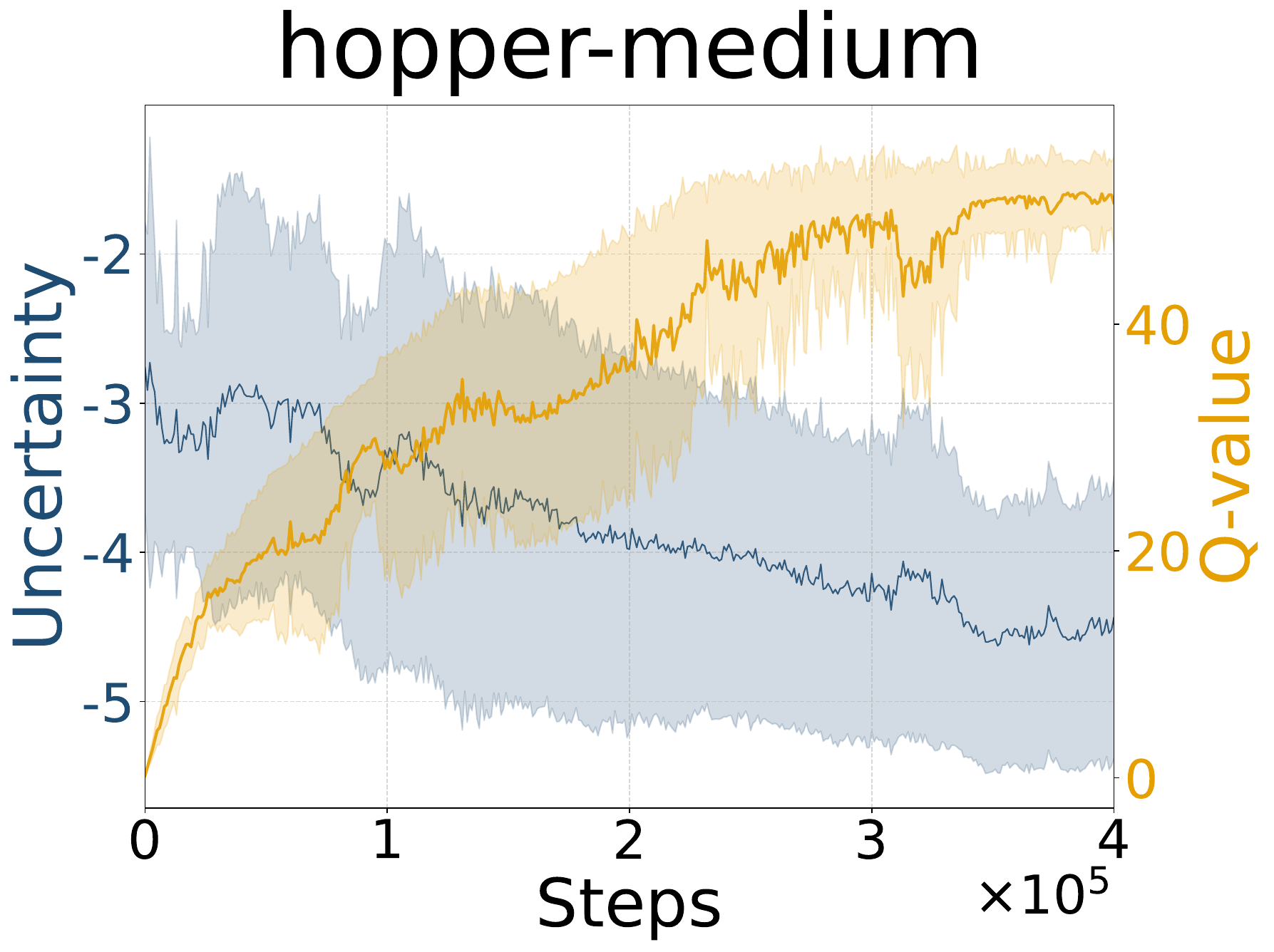}
	\includegraphics[width=0.24\textwidth]{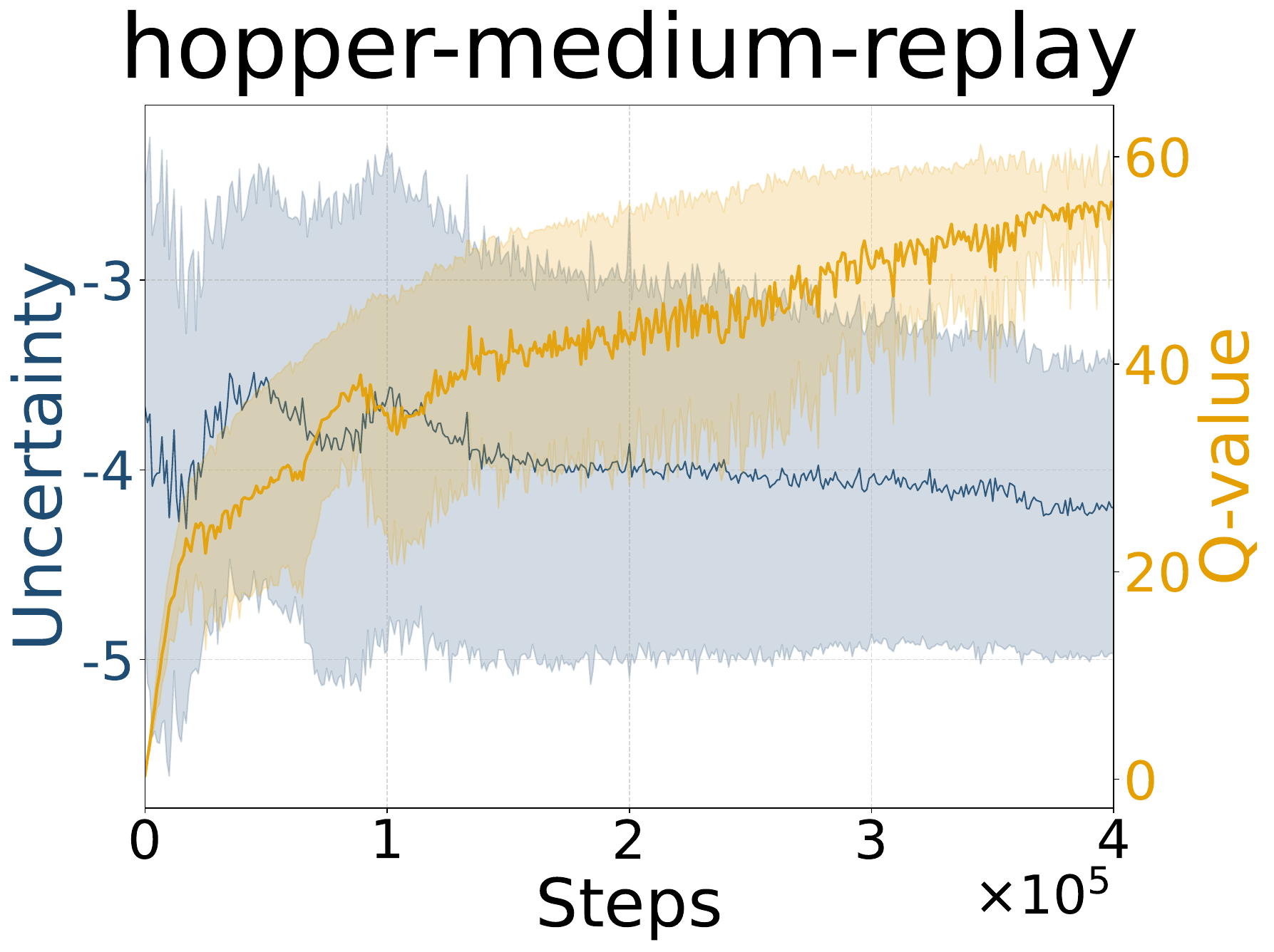}
	\includegraphics[width=0.24\textwidth]{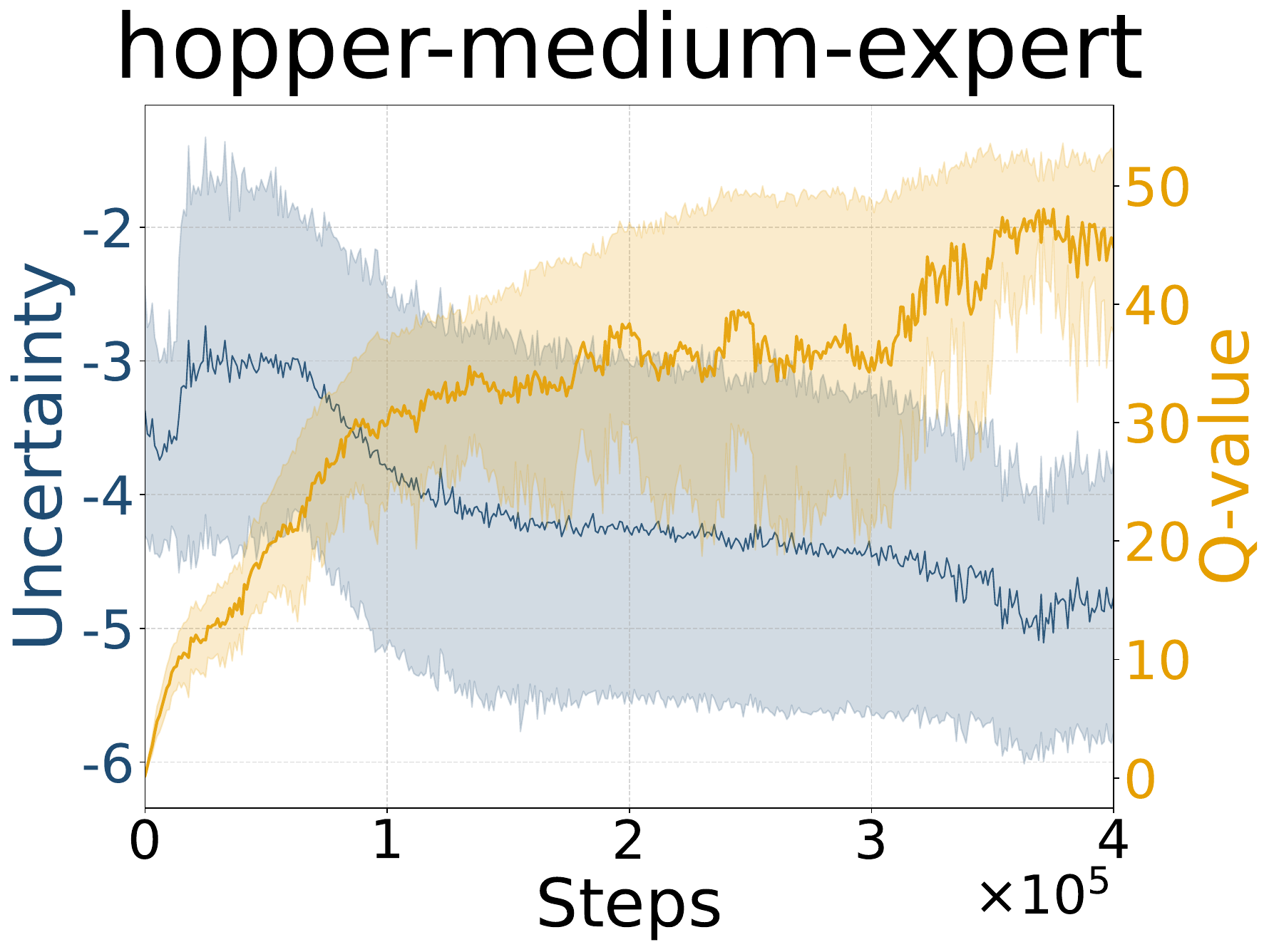}
	\includegraphics[width=0.24\textwidth]{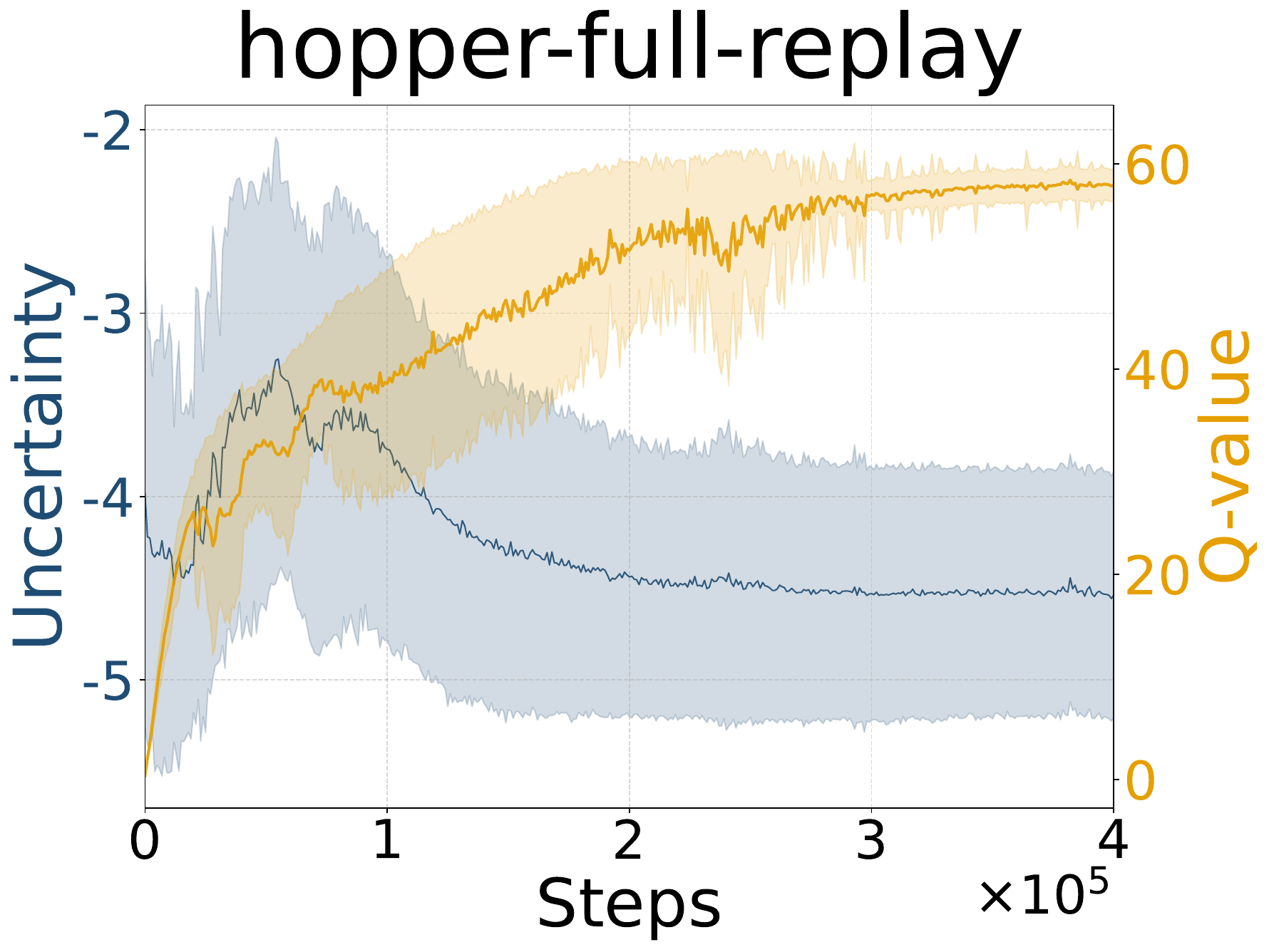}
	\caption{Evolution of the learned $Q$-values and predictive uncertainty for encountered state--action pairs during training. For a given $(s,a)$, uncertainty is quantified as the log standard deviation of next-state predictions over the model ensemble.}
	\label{exp_fig2:uncertainty}
\end{figure*}

\section{Experimental Results}
In this section, we evaluate RRPI on D4RL benchmarks \cite{fu2020d4rl} with three goals. First, we assess overall performance against strong baselines to verify the effectiveness of the proposed robust objective. Second, we analyze uncertainty estimation to examine whether the learned policy behaves conservatively in high-uncertainty regions, which is critical under distribution shift in offline settings. Third, we conduct ablation studies to isolate the contribution of the worst-case optimization and validate that robustness is essential for the gains.
\subsection{Performance}

We compare our method with a broad set of state-of-the-art baselines, including model-free algorithms such as CQL \cite{kumar2020conservative}, DMG \cite{maodoubly}, and EPQ \cite{NEURIPS2024_fdb11be1}, as well as model-based approaches including MOReL \cite{kidambi2020morel}, RAMBO \cite{rigter2022rambo}, PMDB \cite{guo2022model}, and ADM \cite{lin2025anystep}. As shown in Table~\ref{tab:d4rl_performance}, our
method achieves the best average performance on D4RL benchmarks. Reference scores are taken from the original papers; for methods without reported numbers (e.g., ADM and DMG), we reproduce results by following the hyperparameter settings from the respective publications. In particular, compared with PMDB, our method outperforms it on 11 out of 18 environments and remains competitive on the remaining 7. This result suggests that robust optimization can provide better resilience under extreme transition perturbations than approaches that target percentile-based performance.

We attribute the strong performance to two factors. First, the robust objective explicitly optimizes against transition uncertainty, which mitigates compounding model bias when the learned policy induces state-action distributions that deviate from the dataset. Second, the regularized policy updates improve training stability by preventing overly aggressive policy shifts, enabling the algorithm to more effectively exploit limited high-quality data. These properties are reflected by the consistently strong results on the ``medium,'' ``medium-expert,'' and ``full-replay'' datasets. 

\begin{table*}[t!]
	\centering
	\caption{Ablation study results. We report the percentage drop in score ($\downarrow$) and the fold change in standard deviation (shown as $\uparrow/\downarrow$). The symbol ``x'' denotes the multiplicative factor (e.g., $\uparrow$10x indicates a ten-fold increase).}
	\label{tab:ablation_study}
	
	\setlength{\tabcolsep}{15.5pt}
	\adjustbox{width=0.99\textwidth}
	{
		\begin{tabular}{l|cccc}
			\toprule
			\textbf{Dataset} & \textbf{HalfCheetah} & \textbf{Hopper} & \textbf{Walker2d} \\
			\midrule
			
			{Random} 
			
			& $\downarrow$7.4\% ($\uparrow$7.4x) & $\downarrow$71.9\% ($\uparrow$1.4x)   & $\downarrow$4.3\% ($\uparrow$2.5x) \\
			
			{Medium}  
			& $\downarrow$12.4\% ($\uparrow$3.9x)  & $\downarrow$10.9\% ($\uparrow$16.1x)   & $\downarrow$17.1\% ($\uparrow$17.1x) \\
			
			{Expert} 
			& $\downarrow$1.2\% ($\uparrow$2.2x)  & $\downarrow$5.3\% ($\uparrow$7.1x)  & $\downarrow$4.1\% ($\uparrow$6.3x) \\
			
			{Medium-Expert}
			& $\downarrow$7.8\% ($\uparrow$2.4x)  & $\downarrow$10.3\% ($\uparrow$33.7x)    & $\downarrow$14.7\% ($\uparrow$2.7x) \\
			
			{Medium-Replay} 
			& $\downarrow$10.7\% ($\uparrow$2.3x)    & $\downarrow$9.1\% ($\uparrow$25.3x)    & $\downarrow$31.5\% ($\uparrow$3.8x) \\
			
			{Full-Replay} 
			& $\downarrow$4.6\% ($\uparrow$4.2x)   & $\downarrow$7.8\% ($\uparrow$18.1x)  & $\downarrow$17.1\% ($\uparrow$19.0x) \\
			
			\bottomrule
		\end{tabular}
	}
\end{table*}

\subsection{Uncertainty Estimation}
Although RRPI does not require an explicit uncertainty penalty during training, the proposed robust formulation induces an implicit sensitivity to epistemic uncertainty via the inner minimization over the uncertainty set $P$. To make this effect measurable, we report an ensemble-based uncertainty proxy along trajectories generated by the learned policy. Concretely, for each encountered state-action pair $(s,a)$, we evaluate all models in the learned dynamics ensemble and compute the standard deviation of the predicted next states. This scalar summarizes model disagreement and serves as a practical indicator of epistemic uncertainty in model-based offline RL \cite{yu2020mopo}. Uncertainty traces and evolving learned $Q$-values are shown in Figure~\ref{exp_fig2:uncertainty}.

We observe that uncertainty spikes coincide with sharp drops in $Q(s,a)$, and the learned policy progressively avoids high-uncertainty regions, yielding a smoother $Q$-landscape. We attribute this phenomenon to the interaction between worst-case dynamics selection and regularized policy improvement. When the ensemble exhibits high disagreement at $(s,a)$, the uncertainty set contains plausible models that predict substantially different next states and rewards. The inner minimization then selects an adverse transition within $P$, reducing the robust Bellman backup at $(s,a)$ and propagating pessimism to predecessor states. The subsequent KL-regularized policy update limits abrupt distributional shifts while still preferring actions with higher robust $Q$-values, causing the learned policy to down-weight uncertain actions over iterations. Unlike heuristic uncertainty penalties, the reduced values in high-disagreement regions here arise directly from optimizing a principled worst-case objective, providing a coherent explanation for the observed coupling between uncertainty spikes and $Q$-value drops.

\subsection{Ablation Study}
To isolate the effect of robust optimization, we conduct an ablation study in which we remove the worst-case model selection. Specifically, instead of optimizing against the worst-case model in the uncertainty set, the ablated variant randomly samples a single dynamics model from the uncertainty set and performs rollouts under this sampled model. As shown in Table~\ref{tab:ablation_study}, the ablated variant's significant performance degradation validates the robust formulation in RRPI. We attribute this degradation to the inability of random model selection to consistently expose the policy to adverse transitions. Without the inner worst-case optimization, the learned policy can overfit to benign models in the ensemble and become brittle when evaluated under more challenging dynamics, leading to reduced returns and increased variance.

\section{Conclusions}

We study robust model-based offline reinforcement learning by casting transition dynamics as a decision variable within an uncertainty set and optimizing policies against worst-case dynamics. To make this robust formulation practical, we propose Robust Regularized Policy Iteration (RRPI), which replaces the intractable max-min bilevel objective with a tractable KL-regularized surrogate and solves it via iterative policy evaluation and improvement.

Our main contributions are threefold. First, we propose a robust regularized Bellman operator and derive a corresponding soft-greedy policy update that admits efficient implementation. Second, we establish theoretical guarantees for the operator’s contraction and fixed-point properties, as well as the policy iteration’s monotonic improvement and convergence. Third, experiments on D4RL benchmarks demonstrate that RRPI achieves strong performance and improves robustness compared with recent baselines.

Moving forward, we aim to tighten the theoretical-empirical gap in uncertainty estimation and explore the integration of multimodal observations (e.g., vision \cite{SUN2026112800,zhang2026pointcotmultimodalbenchmarkexplicit,zhang2026igasa}) into the RRPI framework to handle increasingly complex decision-making tasks.

%
%

%
%
%
\bibliographystyle{splncs04}
\bibliography{mybib}
\appendix
\setcounter{theorem}{0}
\section{Proofs}
\begin{theorem}
	\label{thm:rrpo}
	The robust regularized Bellman operator $\mathcal{T}$ is a $\gamma$-contraction mapping under the $\|\cdot\|_{\infty}$ norm. Starting from any initial action-value function $Q:S\times A\to\mathbb{R}$ and repeatedly applying $\mathcal{T}$, the resulting sequence converges to a unique fixed point $Q^*$. Moreover, the optimal robust policy for the surrogate objective $\widehat{\eta}(\pi,p,\mu)$ is given by
	\begin{equation}
		\label{eq:opt_policy}
		\pi^*(a|s)\propto \mu(a|s)\exp\!\left(\frac{1}{\alpha}Q^*(s,a)\right).
	\end{equation}
\end{theorem}

\begin{proof}
Let \(Q_1,Q_2:\mathcal{S}\times\mathcal{A}\to\mathbb{R}\) be arbitrary. For any state \(s'\), define
\[
f_Q(s')=\alpha\log\mathbb{E}_{a'\sim\mu(\cdot|s')}\exp\Bigl(\frac{1}{\alpha}Q(s',a')\Bigr).
\]
For a fixed pair \((s,a)\), we have
\[
|\mathcal{T}Q_1(s,a)-\mathcal{T}Q_2(s,a)|
=\gamma\Bigl|\min_{p\in\mathcal{P}}\mathbb{E}_{p}[f_{Q_1}]-\min_{p\in\mathcal{P}}\mathbb{E}_{p}[f_{Q_2}]\Bigr|.
\]
A basic property of the minimum gives
\[
\Bigl|\min_{i}x_i-\min_{i}y_i\Bigr|\le\max_{i}|x_i-y_i|.
\]
Applying this to the finite or compact set \(\mathcal{P}\) (or by a straightforward extension to a family of expectations) yields
\begin{equation*}
	\begin{aligned}
		|\mathcal{T}Q_1(s,a)-\mathcal{T}Q_2(s,a)|
		&\le\gamma\max_{p\in\mathcal{P}}\bigl|\mathbb{E}_{p}[f_{Q_1}-f_{Q_2}]\bigr|
		\\&\le\gamma\max_{p\in\mathcal{P}}\mathbb{E}_{p}\bigl|f_{Q_1}-f_{Q_2}\bigr|
		\\&\le\gamma\max_{s'}|f_{Q_1}(s')-f_{Q_2}(s')|.
	\end{aligned}
\end{equation*}
We now bound \(|f_{Q_1}(s')-f_{Q_2}(s')|\). Let \(\epsilon=\|Q_1-Q_2\|_\infty\). For any \(a'\),
\[
Q_1(s',a')\le Q_2(s',a')+\epsilon.
\]
Then, we have
\[
\exp\Bigl(\frac{1}{\alpha}Q_1(s',a')\Bigr)\le e^{\epsilon/\alpha}\exp\Bigl(\frac{1}{\alpha}Q_2(s',a')\Bigr).
\]
Taking expectation with respect to \(\mu\) and then the logarithm gives
\[
\alpha\log\mathbb{E}_\mu e^{Q_1/\alpha}\le \alpha\log\bigl(e^{\epsilon/\alpha}\mathbb{E}_\mu e^{Q_2/\alpha}\bigr)
= \alpha\log\mathbb{E}_\mu e^{Q_2/\alpha}+\epsilon.
\]
The reverse inequality follows similarly by swapping \(Q_1\) and \(Q_2\). Consequently,
\[
|f_{Q_1}(s')-f_{Q_2}(s')|\le\epsilon.
\]
Thus for all \((s,a)\),
\[
|\mathcal{T}Q_1(s,a)-\mathcal{T}Q_2(s,a)|\le\gamma\|Q_1-Q_2\|_\infty,
\]
and taking the supremum over \((s,a)\) yields \(\|\mathcal{T}Q_1-\mathcal{T}Q_2\|_\infty\le\gamma\|Q_1-Q_2\|_\infty\). Hence \(\mathcal{T}\) is a \(\gamma\)-contraction. By the Banach fixed‑point theorem, there exists a unique fixed point \(Q^*\) satisfying \(\mathcal{T}Q^*=Q^*\).

We first recall the well‑known duality relation
\[
\alpha\log\mathbb{E}_{a'\sim\mu}\exp\Bigl(\frac{1}{\alpha}Q(s',a')\Bigr)
= \max_{\pi(\cdot|s')}\Bigl\{\mathbb{E}_{a'\sim\pi}[Q(s',a')]-\alpha D_{\mathrm{KL}}\bigl(\pi\|\mu\bigr)\Bigr\},
\]
and the maximiser is attained at
\[
\pi^*(a'|s')\propto\mu(a'|s')\exp\Bigl(\frac{1}{\alpha}Q(s',a')\Bigr).
\tag{*}
\]

For any policy \(\pi\), define the robust regularized action‑value function \(Q_\mathcal{P}^\pi\) as the unique fixed point of the operator
\[
\mathcal{T}^\pi Q(s,a)=r(s,a)+\gamma\min_{p\in\mathcal{P}}\mathbb{E}_{s'\sim p}\Bigl[\mathbb{E}_{a'\sim\pi}\bigl[Q(s',a')-\alpha\log\frac{\pi(a'|s')}{\mu(a'|s')}\bigr]\Bigr].
\]
One easily verifies that \(\mathcal{T}^\pi\) is also a \(\gamma\)-contraction (the proof follows the same steps as for \(\mathcal{T}\), using the fact that the term inside the expectation is a function of \(s'\) only). Therefore \(Q_\mathcal{P}^\pi\) is well defined and satisfies the Bellman equation
\[
Q_\mathcal{P}^\pi(s,a)=r(s,a)+\gamma\min_{p\in\mathcal{P}}\mathbb{E}_{s'\sim p}\Bigl[\mathbb{E}_{a'\sim\pi}\bigl[Q_\mathcal{P}^\pi(s',a')-\alpha\log\frac{\pi(a'|s')}{\mu(a'|s')}\bigr]\Bigr].
\]

Now let \(Q^*\) be the fixed point of \(\mathcal{T}\) and define \(\pi^*\) by (*). By the duality relation,
\[
\alpha\log\mathbb{E}_\mu e^{Q^*/ \alpha}
= \mathbb{E}_{a'\sim\pi^*}[Q^*(s',a')]-\alpha D_{\mathrm{KL}}(\pi^*\|\mu)\qquad\forall s'.
\]
Substituting this into \(\mathcal{T}Q^*=Q^*\) gives
\[
Q^*(s,a)=r(s,a)+\gamma\min_{p\in\mathcal{P}}\mathbb{E}_{s'\sim p}\Bigl[\mathbb{E}_{a'\sim\pi^*}[Q^*(s',a')]-\alpha D_{\mathrm{KL}}(\pi^*\|\mu)\Bigr]
= \mathcal{T}^{\pi^*}Q^*(s,a).
\]
Hence \(Q^*\) is a fixed point of \(\mathcal{T}^{\pi^*}\); by uniqueness, \(Q^*=Q_\mathcal{P}^{\pi^*}\).

Take an arbitrary policy \(\pi\). From the duality relation we have the inequality
\[
\alpha\log\mathbb{E}_\mu e^{Q_\mathcal{P}^\pi/ \alpha}
\ge \mathbb{E}_{a'\sim\pi}[Q_\mathcal{P}^\pi(s',a')]-\alpha D_{\mathrm{KL}}(\pi\|\mu).
\]
Applying this pointwise yields
\begin{equation*}
	\begin{aligned}
		\mathcal{T}Q_\mathcal{P}^\pi(s,a)
		&\ge r(s,a)+\gamma\min_{p\in\mathcal{P}}\mathbb{E}_{s'\sim p}\Bigl[\mathbb{E}_{a'\sim\pi}[Q_\mathcal{P}^\pi(s',a')]-\alpha D_{\mathrm{KL}}(\pi\|\mu)\Bigr]
		\\& = \mathcal{T}^\pi Q_\mathcal{P}^\pi(s,a)=Q_\mathcal{P}^\pi(s,a).
	\end{aligned}
\end{equation*}
Thus \(\mathcal{T}Q_\mathcal{P}^\pi\ge Q_\mathcal{P}^\pi\) componentwise. Because \(\mathcal{T}\) is monotone (if \(Q_1\ge Q_2\) then \(\mathcal{T}Q_1\ge\mathcal{T}Q_2\)) and a contraction, iterating \(\mathcal{T}\) starting from \(Q_\mathcal{P}^\pi\) gives an increasing sequence that converges to the unique fixed point \(Q^*\); consequently \(Q^*\ge Q_\mathcal{P}^\pi\) everywhere.

Finally, for the initial state distribution \(\rho_0\) we have
\begin{equation*}
	\begin{aligned}
		\hat{\eta}(\pi^*,p^*,\mu)
		&= \mathbb{E}_{\rho_0,\pi^*}\Bigl[Q^*(s_0,a_0)-\alpha D_{\mathrm{KL}}(\pi^*\|\mu)\Bigr]
		\\& \ge \mathbb{E}_{\rho_0,\pi}\Bigl[Q_\mathcal{P}^\pi(s_0,a_0)-\alpha D_{\mathrm{KL}}(\pi\|\mu)\Bigr]
		= \hat{\eta}(\pi,p_\pi,\mu),
	\end{aligned}
\end{equation*}
where \(p^*\) and \(p_\pi\) denote the worst‑case transitions inside \(\mathcal{P}\) for \(\pi^*\) and \(\pi\) respectively. Since \(\pi\) is arbitrary, \(\pi^*\) attains the maximum of the robust regularized objective. 

\end{proof}

\begin{theorem}
	\label{thm:irpo}
	Starting from any stochastic initial policy $\pi_0:S\to\Delta(A)$ with full support (i.e., $\pi_0(a|s)>0$ for all $(s,a)\in S\times A$), we generate a sequence $\{\pi_i\}_{i\ge 0}$ by repeatedly solving the robust sregularized ubproblem with the reference policy set to the previous iterate:
	\begin{equation}
		\pi_{i+1}\in\arg\max_{\pi\in\Pi}\min_{p\in P}\widehat{\eta}(\pi,p,\mu=\pi_i).
	\end{equation}
	Then the sequence monotonically improves the original objective $J(\pi)\:=min_{p\in P}\eta(\pi,p)$ in the sense that $J(\pi_{i+1})\ge J(\pi_i)$ for all $i\ge 0$.
	Furthermore, the sequence converges to an optimal robust policy for the original (non-regularized) problem: for any state $s$ and actions $a,a'\in A$ satisfying
	$\lim_{i\to\infty}Q_{P}^{\pi_i}(s,a)>\lim_{i\to\infty}Q_{P}^{\pi_i}(s,a')$,
	we have $\pi_i(a\mid s)/\pi_i(a'\mid s)\to\infty$ as $i\to\infty$, where $Q_{P}^{\pi_i}$ denotes the action-value function under the worst-case dynamics in $P$.
\end{theorem}

\begin{proof}
	
Let \(J_i := J(\pi_i) = \min_{p \in \mathcal{P}} \eta(\pi_i, p)\). By definition, \(\pi_{i+1}\) maximizes the worst‑case regularized return, hence
\[
\min_{p \in \mathcal{P}} \hat{\eta}(\pi_{i+1}, p, \pi_i) \;\ge\; \min_{p \in \mathcal{P}} \hat{\eta}(\pi_i, p, \pi_i).
\]
For \(\pi = \pi_i\) the KL‑regularization term vanishes, so
\[
\hat{\eta}(\pi_i, p, \pi_i) = \eta(\pi_i, p) \quad \text{for every } p,
\]
and consequently
\[
\min_{p \in \mathcal{P}} \hat{\eta}(\pi_i, p, \pi_i) = \min_{p \in \mathcal{P}} \eta(\pi_i, p) = J_i.
\]
Thus we have
\[
\min_{p \in \mathcal{P}} \hat{\eta}(\pi_{i+1}, p, \pi_i) \ge J_i. \tag{A}
\]
On the other hand, because the KL divergence is non‑negative,
\[
\hat{\eta}(\pi_{i+1}, p, \pi_i) \le \eta(\pi_{i+1}, p) \quad \text{for every } p,
\]
and taking the minimum over \(p\) preserves the inequality:
\[
\min_{p \in \mathcal{P}} \hat{\eta}(\pi_{i+1}, p, \pi_i) \le \min_{p \in \mathcal{P}} \eta(\pi_{i+1}, p) = J_{i+1}. \tag{B}
\]
Combining (A) and (B) yields \(J_i \le J_{i+1}\) for all \(i\), establishing monotonic improvement.

For each iteration \(i\), let \(Q_i^*=Q^{\pi_i}_P\) denote the action-value function obtained from the robust regularized Bellman operator with reference policy $\pi_{i}$; this is the unique fixed point of the operator \(\mathcal{T}_{\pi_i}\) defined in Theorem~1. Theorem~1 also gives the form of the optimal policy for the subproblem:
\[
\pi_{i+1}(a|s) \propto \pi_i(a|s) \exp\!\left( \frac{1}{\alpha} Q_i^*(s,a) \right). \tag{*}
\]

We know that $\{J(\pi_i)\}$ is a non-decreasing sequence bounded above (the return is bounded under the usual assumptions), therefore it converges. Consequently, the value functions $\{Q_i^*\}$ converge pointwise to some limit $\bar{Q}$. Now fix a state \(s\) and two actions \(a,a'\) such that \(\bar{Q}(s,a) > \bar{Q}(s,a')\). Set \(\delta = \bar{Q}(s,a) - \bar{Q}(s,a') > 0\) and choose \(\varepsilon = \delta/4\). Since \(Q_i^* \to \bar{Q}\) pointwise, there exists an index \(N\) such that for all \(i \ge N\),
	\[
	|Q_i^*(s,a) - \bar{Q}(s,a)| < \varepsilon, \qquad |Q_i^*(s,a') - \bar{Q}(s,a')| < \varepsilon.
	\]
	Then for every \(i \ge N\),
	\[
	Q_i^*(s,a) - Q_i^*(s,a') \;\ge\; (\bar{Q}(s,a)-\varepsilon) - (\bar{Q}(s,a')+\varepsilon) \;=\; \delta - 2\varepsilon \;=\; \frac{\delta}{2} \;>\; 0.
	\]
	Let \(c = \delta/2 > 0\), we have \(Q_i^*(s,a) - Q_i^*(s,a') \ge c\) for all \(i \ge N\).
	
	From the update rule \((*)\) we obtain a recurrence for the ratio of the policy probabilities:
	\[
	\frac{\pi_{i+1}(a|s)}{\pi_{i+1}(a'|s)}
	= \frac{\pi_i(a|s)}{\pi_i(a'|s)} \;
	\exp\!\left( \frac{Q_i^*(s,a) - Q_i^*(s,a')}{\alpha} \right).
	\]
	Applying this repeatedly for \(i = N, N+1, \dots, k-1\) yields
	\begin{equation*}
		\begin{aligned}
			\frac{\pi_k(a|s)}{\pi_k(a'|s)}
			&= \frac{\pi_N(a|s)}{\pi_N(a'|s)} \;
			\exp\!\left( \frac{1}{\alpha} \sum_{i=N}^{k-1} \bigl( Q_i^*(s,a) - Q_i^*(s,a') \bigr) \right)
			\\&\ge \frac{\pi_N(a|s)}{\pi_N(a'|s)} \;
			\exp\!\left( \frac{(k-N)c}{\alpha} \right).
		\end{aligned}
	\end{equation*}
	
	Because \(\pi_0\) has full support and the updates preserve positivity, \(\pi_N(a|s) > 0\) and \(\pi_N(a'|s) > 0\), hence the initial ratio \(\frac{\pi_N(a|s)}{\pi_N(a'|s)}\) is a finite positive number. Letting \(k \to \infty\), the exponential factor tends to infinity, and therefore
	\[
	\frac{\pi_k(a|s)}{\pi_k(a'|s)} \to \infty.
	\]
	
	This shows that in the limit the policy concentrates all probability mass on actions with the highest limiting value, which is exactly the behaviour of an optimal robust policy for the original (non‑regularized) problem. 
\end{proof}

\end{document}